\documentclass[conference]{IEEEtran}
\usepackage{cite}
\usepackage{amsmath,amssymb}
\usepackage{amsthm}
\usepackage{graphicx}
\usepackage{booktabs}
\usepackage[hidelinks]{hyperref}
\usepackage{microtype}

\newtheorem{definition}{Definition}
\newtheorem{lemma}{Lemma}
\newtheorem{proposition}{Proposition}

\newcommand{\reach}{\mathsf{reach}}
\newcommand{\isactive}{\mathsf{active}}
\newcommand{\revoked}{\mathsf{REVOKED}}
\newcommand{\valid}{\mathsf{valid}}
\newcommand{\intent}{T_{\mathsf{intent}}}
\newcommand{\overrev}{\mathsf{OverRev}}
\newcommand{\underrev}{\mathsf{UnderRev}}

\newcommand{\roots}{\mathsf{Roots}}
\newcommand{\method}{\textsc{VERA}}

\begin{document}

\title{\method: Authority-Preserving Edge Revocation for\\Federated AI-Agent Workflows}
\author{
\IEEEauthorblockN{Lifei Liu}
\IEEEauthorblockA{Independent Researcher\\Seattle, WA, USA\\lliu.lifei@gmail.com}
\and
\IEEEauthorblockN{Haoran Yu}
\IEEEauthorblockA{Independent Researcher\\Seattle, WA, USA\\haoranyu889@gmail.com}
\and
\IEEEauthorblockN{Xiaochong Jiang}
\IEEEauthorblockA{Independent Researcher\\Seattle, WA, USA\\jiang.xiaoc@northeastern.edu}
}
\maketitle

\begin{abstract}
Modern agent frameworks compose planners, tool agents, remote services, and
shared specialists into runtime delegation graphs, but their revocation APIs
still resemble token or subtree invalidation. When one delegation is withdrawn,
the runtime must know which agents lose authority while independently authorized
agents keep working. We study this \emph{authority consistency} problem and
introduce \method{} (Verifiable Edge Revocation for Agents), a verifier-checkable
revocation contract and API emitted by agent-runtime adapters as signed evidence.
Under disjunctive authority, revoking edge $e$ invalidates exactly
$\intent(e,G)=\reach(G)\setminus\reach(G\setminus\{e\})$, the agents whose every
authorizing root path used $e$. Used as a contract, this target exposes two
runtime failures: tree cascades over-revoke shared agents, while
deployer-scoped cascades under-revoke cross-domain descendants. In a LangGraph
framework-replay study, 25 workflow/input cells repeated 20 times yield 500
compiled-framework traces and 2,000 valid signed delegation decisions; 13/25
cells contain runtime multi-parent sharing and 8/25 contain cross-deployer
sharing. Signed replay verifies 500/500 target proofs, preserves all 320
alternate-parent shared-agent cases that tree cascade revokes, and rejects
unauthorized signers and omission attacks. Baseline replay over 1,960 edge
withdrawals shows that holder/node and tree-style targets cannot express this
behavior. We further validate schema portability on A2A, AutoGen, and CrewAI
artifacts: nine traces, including five executable CrewAI workflows, yield
53 signed delegation events that pass schema and signature checks.
\end{abstract}

\begin{IEEEkeywords}
AI agents, multi-agent systems, trustworthy AI, delegation, authorization,
revocation, agent workflows
\end{IEEEkeywords}

\section{Introduction}

AI-agent workflows are becoming multi-principal systems. Frameworks such as
AutoGen~\cite{wu2023autogen}, CrewAI~\cite{crewai2024}, and
LangGraph~\cite{langgraph2026} compose planners, tool agents, remote
specialists, and shared service agents; A2A~\cite{google2026a2a} and
MCP~\cite{mcp2025auth} make those handoffs cross process and organizational
boundaries. Recent agent-identity work argues that these tools need
authenticated, authorized, and auditable delegated
authority~\cite{south2025delegation,south2025identity,prakash2026aip}.
This paper asks the next workflow question: what revocation contract should an
agent runtime or middleware expose when one delegation relationship is withdrawn?

The ambiguity is semantic. Consider a node $C$ with two delegations $A\to C$ and
$B\to C$: if $A$ revokes only its edge, a disjunctive policy keeps $C$ authorized
through $B\to C$, so a tree cascade from $C$ over-revokes $C$ and its descendants.
Conversely, if $A\to C\to D$ crosses into another deployer domain, a cascade
restricted to $A$'s registry may revoke $C$ but miss $D$ (under-revocation). We
therefore argue that revocation must specify its semantic target before choosing a
propagation mechanism. Under disjunctive authority---a node is authorized iff some
valid root-to-node path survives---revoking edge $e$ should invalidate exactly the
nodes that lose all root paths when $e$ is removed,
\[
  \intent(e,G) = \reach(G) \setminus \reach(G \setminus \{e\}),
\]
the nodes edge-dominated by $e$. Our contribution is not to introduce a new
dominator algorithm; it is to turn this correctness oracle into an executable
agent-workflow contract: the runtime emits edge identifiers, signed delegation
events, revocation envelopes, and enough evidence for a relying party to check
the target without re-querying a central authorization service. We then measure
when common target languages are unsafe or disruptive.
Figure~\ref{fig:witnesses} shows the failure on concrete multi-agent handoff
graphs: a shared specialist should survive if another workflow still delegates to
it, while a remote descendant should be revoked if its only authority runs
through the withdrawn edge.

\paragraph{Contributions.}
\begin{itemize}
\item \textbf{\method{} contract and API.} A scoped signed-DAG model and
reference schema for federated AI-agent delegation, where the revocation API
names delegation edges and computes the authority-preserving target
$\reach(G)\setminus\reach(G\setminus\{e\})$.
\item \textbf{Framework-grounded evidence.} Executed LangGraph traces are
normalized into signed delegation decisions and replayed against exact,
holder/node, tree, deployer-scoped, and full-graph revalidation targets. A2A,
AutoGen, and CrewAI artifacts exercise the same schema beyond LangGraph.
\item \textbf{Verifier-checkable artifact.} Ed25519-signed edges, bilateral
in-edge commitments, a history-binding revocation log, the \method{} Python API,
and a CLI verifier check the batch target
$\reach(G)\setminus\reach(G\setminus R)$ from signed survivor/cut evidence
without storing the full graph.
\end{itemize}
Propagation/delivery is an orthogonal layer and out of scope here; we fix the
target semantics it must carry.

\begin{figure*}[t]
\centering
\includegraphics[width=0.80\textwidth,trim=0 145 0 0,clip]{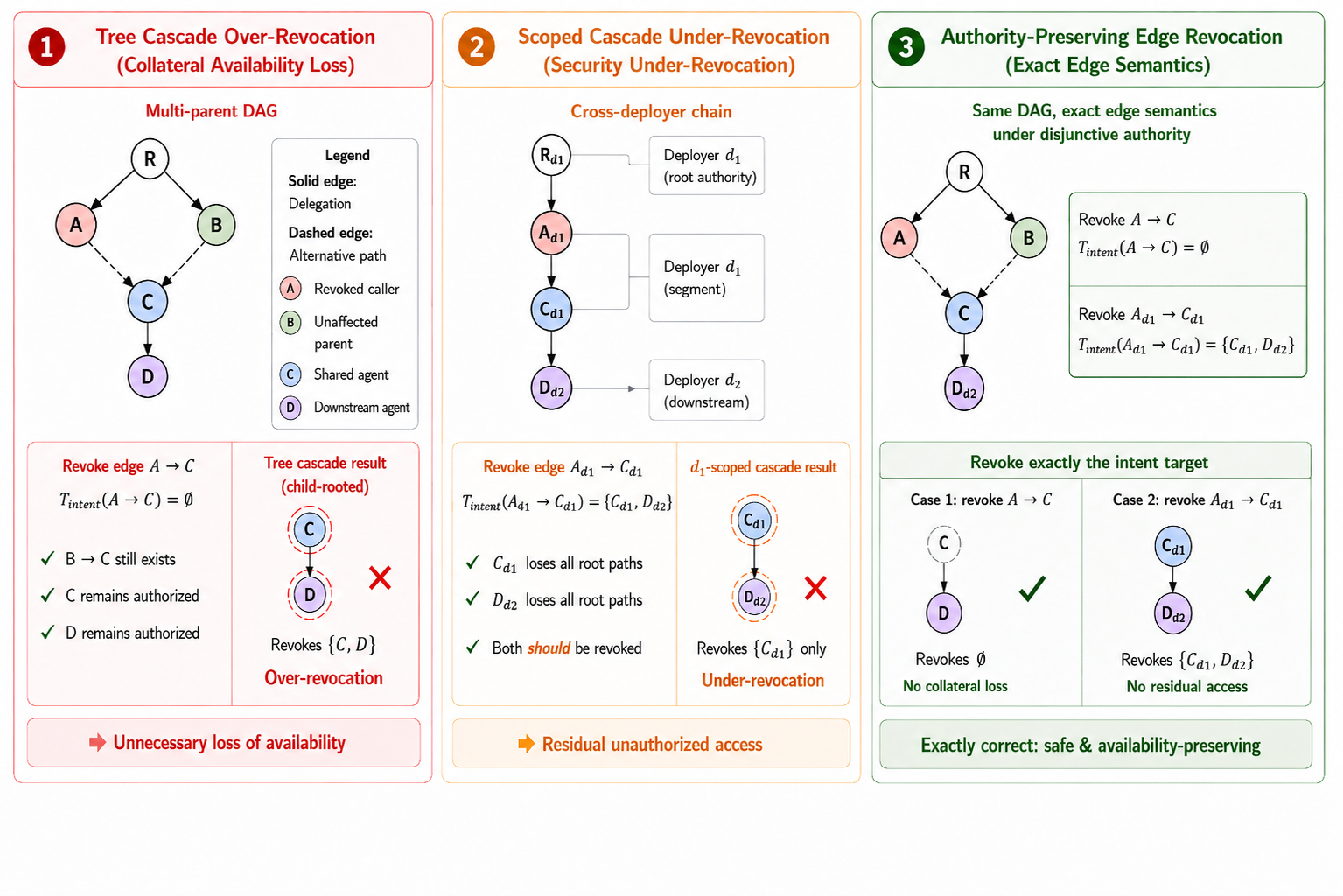}
\caption{Motivation and semantic insight. Tree cascades can over-revoke shared
agents, deployer-scoped cascades can under-revoke cross-domain descendants, and
edge revocation preserves exactly the agents that retain an alternate
authorizing path.}
\label{fig:witnesses}
\end{figure*}

\paragraph{Scope.}
We do not build Byzantine propagation, transparency logs, or a production
platform. \method{} is a revocation adapter/middleware contract for existing
orchestrators that emit signed events. Disjunctive authority is one common
policy; other authority policies require different target functions.

\section{Background and Motivation}

\subsection{From Chains to Multi-Parent DAGs}

Agent-identity and delegated-authorization proposals often emit chains: a user
delegates to an agent, that agent invokes a tool or another agent, and a verifier
checks signed or token-bound authority statements. Agent runtimes also create
multi-parent patterns: a task agent may be authorized by scheduling and billing
agents, a shared specialist may be admitted by multiple orchestrators, and a
cross-deployer task may pass from one organization to another. These are
orchestration-layer relationships, so we model them as a directed acyclic graph.
The multi-parent case, $|\mathsf{parents}(v)| > 1$, is where revocation semantics
become nontrivial.

\subsection{Why Existing Revocation Mechanisms Are Not Enough}

Classical PKI mechanisms (CRLs~\cite{RFC5280}, OCSP~\cite{RFC6960},
CRLite~\cite{crlite2017}) revoke credential bindings, not individual delegation
relationships. Macaroons~\cite{birgisson2014macaroons} and
SPIFFE~\cite{spiffe2018} rely on attenuation, re-contact, or short lifetimes;
DW-RBAC~\cite{wainer2007delegation} assumes a centralized single-domain engine.
Agent delegation instead needs signed evidence that lets relying parties verify
authority and revocation without a central registry at decision time.
Table~\ref{tab:mechanisms} summarizes the gap: adjacent mechanisms authenticate
credentials, endpoints, or delegated identity, but none defines the
authority-preserving target for revoking one edge in a multi-parent,
cross-deployer agent DAG.

\begin{table}[t]
\centering
\caption{Adjacent mechanisms authenticate credentials, agent identity, or
communication endpoints, but do not define the exact authority loss caused by one
revoked edge in a multi-parent agent DAG.}
\label{tab:mechanisms}
\small
\resizebox{\columnwidth}{!}{%
\begin{tabular}{llcc}
\toprule
Mechanism & Main target & Multi- & Exact edge \\
          &             & parent & target \\
\midrule
CRL / OCSP~\cite{RFC5280,RFC6960} & certificate & no & no \\
OAuth / MCP~\cite{RFC7009,mcp2025auth} & token/scope & no & no \\
Agent IAM~\cite{south2025delegation,south2025identity} & agent identity & partial & no \\
AIP~\cite{prakash2026aip} & delegation token & chain & no \\
A2A~\cite{google2026a2a} & agent endpoint & yes & no \\
APS~\cite{pidlisnyi2026aps} & subtree (tree) & no & no \\
\textbf{This paper} & \textbf{delegation edge} & \textbf{yes} & \textbf{yes} \\
\bottomrule
\end{tabular}
}
\end{table}

\subsection{Two Failure Modes}

Figure~\ref{fig:witnesses} illustrates the two core failures.
\emph{Over-revocation:} in root$\to A,B$; $A,B\to C$; $C\to D$, revoking $A\to C$
leaves $C$ authorized through $B\to C$, yet a tree cascade from $C$ revokes
$\{C,D\}$, causing collateral disruption. \emph{Under-revocation:} in a
cross-deployer chain root$(d_1)\to A(d_1)\to C(d_1)\to D(d_2)$, revoking $A\to C$
should remove both $C$ and $D$, but a $d_1$-scoped cascade sees $C$ and misses the
cross-domain descendant $D$.

\section{Model and Semantics}

\subsection{SIGNED-DAG Delegation Model}

\begin{definition}[SIGNED-DAG]
A SIGNED-DAG is a tuple $G=(V,E,\sigma)$ where $V$ is a finite set of agent
nodes, $E \subseteq V \times V$ is a set of directed delegation edges, and
$\sigma$ maps each edge to a deployer signature over the edge descriptor. Each
node $v$ has a deployer label $\mathsf{dep}(v)$ and public key
$\mathsf{pk}(v)$. An edge $e=(p,c)$ means that parent $p$ delegates authority
to child $c$. The graph is acyclic.
\end{definition}

The DAG is an \emph{execution-evidence} DAG. Retries, recurring workflows, and
cyclic orchestrator code are represented as successive invocation events or
epochs, so each signed replay view remains acyclic.

Roots are trust anchors known to verifiers out of band. We write
$\roots(G)$ for the root set and $\reach(G)$ for the nodes reachable from any
root by following directed edges in $G$.

The signatures in $\sigma$ authenticate the graph. They are not what makes the
edge-dependence result true; the result is graph reachability. Signatures make
the result enforceable without contacting a central authority, because a
relying party can check the authenticity of edges and revocation events locally.

\subsection{Revocation Log}

Let $\mathcal{R}$ be a signed revocation log containing pairs $(e,t_e)$, where
$e$ is an edge identifier and $t_e$ is the revocation time. The edge predicate
is:
\[
  \revoked(e,t) \Leftrightarrow
  \exists (e',t_e) \in \mathcal{R}: e'=e \wedge t_e \leq t.
\]
Let $\mathcal{R}^t=\{e:\revoked(e,t)\}$ be the set of edges revoked by time
$t$.

\begin{definition}[Path validity]
A root-to-node path $\pi=\langle e_0,\ldots,e_k\rangle$ is valid at time $t$
iff every edge is authenticated and unrevoked:
\[
  \valid(\pi,\mathcal{R},t) \Leftrightarrow
  \forall i \in [0,k].\; \neg \revoked(e_i,t) \wedge \mathsf{sig\_ok}(e_i).
\]
\end{definition}

\subsection{Disjunctive Authority}

\begin{definition}[Disjunctive authority]
A node $v$ is active at time $t$ iff at least one valid root-to-$v$ path
exists:
\[
  \isactive(v,G,\mathcal{R}^t) \Leftrightarrow
  v \in \reach(G \setminus \mathcal{R}^t).
\]
\end{definition}

This policy says that independent parent delegations are alternatives: a node
retains authority if one valid parent path survives. The policy matches
settings where multiple orchestrators can independently authorize a shared
agent. It does not cover threshold approval or all-parent approval; those
policies require different target functions.

\subsection{Threat Model and Trust Assumptions}
\label{sec:threat}

The contract is verifier-independent only relative to an authenticated prefix of
the signed delegation graph and revocation log. A relying party pins deployer
public keys out of band, verifies each edge and revocation envelope locally, and
then checks the target without querying a central authorization service. The
same paragraph-level assumptions are summarized in Table~\ref{tab:threat}.

\begin{table}[t]
\centering
\caption{Threat assumptions and contract boundary.}
\label{tab:threat}
\resizebox{\columnwidth}{!}{%
\begin{tabular}{@{}lll@{}}
\toprule
Condition & Treatment & Boundary \\
\midrule
Forged delegation or revocation & reject by pinned deployer key & covered \\
Unauthorized revoker & parent domain must sign envelope & covered \\
Omitted in-edge to target node & signed in-edge commitment must match & covered \\
Omitted reachable node & federation-signed roster must match ($T\cup W=B$) & covered \\
Log rollback or substitution & hash-bound monotone checkpoint detects mismatch & covered after observation \\
Delayed or suppressed update & propagation budgets $\delta_{\mathsf{prop}}$ & out of scope \\
Byzantine deployer or privacy leak & not solved; disclosure is measured & out of scope \\
\bottomrule
\end{tabular}}
\end{table}

Thus Lemma~\ref{lem:edgedom} is a correctness statement over the
\emph{observed, authenticated} graph: $\reach(G)\setminus\reach(G\setminus R)$
is exact for the graph and revocation set the verifier has seen. The prototype
(Section~\ref{sec:prototype}) demonstrates signature checks, rollback detection,
and authenticated adjacency completeness. It does not solve propagation
freshness, transparency-log anchoring, Byzantine deployers who sign conflicting
evidence, or confidentiality beyond the disclosure measurements.

\section{Revocation Targets}

\subsection{Design Requirements}
\label{sec:requirements}

Before defining the target we state the requirements a revocation mechanism for
complex multi-parent agent systems must satisfy; the rest of the paper is
organized around discharging them.

\noindent\textbf{R1 (Edge-level target).} Revocation must name the withdrawn
delegation edge, not a whole credential, principal, or subtree.
\noindent\textbf{R2 (Multi-parent preservation).} A node with an alternate
authorizing root path must stay active.
\noindent\textbf{R3 (Cross-domain completeness).} A node that loses all
authorizing paths must be revoked even when the loss crosses deployer boundaries.
\noindent\textbf{R4 (Verifier independence).} A relying party must check the
target from signed evidence it already holds.
\noindent\textbf{R5 (Batch revocation).} Simultaneous revocations must yield the
exact joint target, not necessarily the union of single-edge targets.

The disjunctive-authority target defined next is the policy that meets R1--R3;
the verifier-independent prototype (Section~\ref{sec:prototype}) discharges R4;
and batch non-compositionality (Section~\ref{sec:batch}) addresses R5.

\subsection{Semantic Target for Edge Revocation}

\begin{definition}[Edge-removal intent]
For a single edge $e \in E$, the intended target set under disjunctive
authority is:
\[
  \intent(e,G) = \reach(G) \setminus \reach(G \setminus \{e\}).
\]
For an edge set $R\subseteq E$, we use
$\intent(R,G)=\reach(G)\setminus\reach(G\setminus R)$.
\end{definition}

This is the set of nodes that had authority before $e$ was removed and do not
have authority after $e$ is removed. It is not the subtree below the child of
$e$, because descendants may retain authority through alternate paths.

\begin{lemma}[Edge-dependence characterization]
\label{lem:edgedom}
For every reachable node $v$ and edge $e$,
\[
  v \in \intent(e,G) \iff
  \text{every root-to-}v\text{ path contains }e.
\]
\end{lemma}

\begin{proof}
If $v \in \intent(e,G)$, then $v$ is reachable in $G$ but unreachable after
removing $e$. Therefore no root-to-$v$ path avoiding $e$ exists, so every
root-to-$v$ path contains $e$. Conversely, if every root-to-$v$ path contains
$e$, then removing $e$ destroys all root-to-$v$ paths. Since $v$ was reachable
before removal and is unreachable after removal, $v \in \intent(e,G)$.
\end{proof}

The target can be computed in $O(|V|+|E|)$ time by two forward reachability
traversals: one on $G$ and one on $G \setminus \{e\}$. For many queries,
standard dominator data structures can precompute equivalent edge-dominance
information. We use this established machinery as the correctness oracle for the
\emph{target language} that signed agent protocols must carry. Edge identifiers,
signed revocation envelopes, in-edge completeness commitments, and proof-carrying
target sets let a verifier check this target without an online central
authorization service.

The contract is meaningful only if a relying party can realize it locally. The
following proposition states the correctness guarantee that the threat model
(Section~\ref{sec:threat}) and prototype (Section~\ref{sec:prototype}) are built
to provide; it formalizes requirement~R4.

\begin{proposition}[Local observed-prefix exactness]
\label{prop:localsafety}
Let a relying party hold an authenticated view $G_v$ together with a
federation-signed snapshot manifest committing to the roster
$B=\reach(G)$ of all root-reachable nodes for the relevant permission. Let $R$
be the revocations observed as an authenticated log prefix. Suppose the verifier
checks sets $T$ and $W$ such that (i)~all used edges and revocations verify under
pinned keys; (ii)~each $w\in W$ has a revocation-free signed root path in $G_v$;
(iii)~each $t\in T$ is non-root and carries a signed complete accepted-in-edge
commitment, with every unrevoked committed in-edge originating in $T$; and
(iv)~$T\cap W=\emptyset$ and $T\cup W=B$. Then
\[
  T=\intent(R,G)
\]
for that observed prefix---i.e.\ the locally checked target equals the intended
target, with no online query to any authority.
\end{proposition}

\begin{proof}
By (i)--(ii), every $w\in W$ has an authenticated revocation-free root path, so
$W\subseteq\reach(G\setminus R)$. No $t\in T$ can have such a path: since $t$ is
not a root and $G$ is acyclic, any root path to $t$ has a first edge entering
$T$ from outside $T$, but (iii) says every unrevoked committed in-edge to a
target node originates inside $T$. Thus
$T\cap\reach(G\setminus R)=\emptyset$. For any
$u\in\reach(G\setminus R)$, the manifest gives $u\in B=T\cup W$; since
$u\notin T$, $u\in W$. Hence $\reach(G\setminus R)\subseteq W$, so
$W=\reach(G\setminus R)$ and
$T=B\setminus W=\reach(G)\setminus\reach(G\setminus R)=\intent(R,G)$.
Exactness is relative to the observed prefix $R$;
revocations not yet observed are the concern of the propagation layer and are
bounded by that layer's stale-authority budget.
\end{proof}

\subsection{Approximation Mechanisms}

We compare mechanisms by the set of nodes they revoke for an edge-removal
request.

\paragraph{Edge target.}
The exact edge target computes $\intent(e,G)$ and is the reference target for all
baseline comparisons.

\paragraph{Tree cascade.}
Tree cascade revokes the child of $e$ and all descendants reachable by forward
traversal. It is exact on trees, but over-revokes on DAGs whenever a descendant
has a surviving alternate root path. This is not a synthetic strawman: it is the
natural generalization of subtree revocation in tree-shaped delegation models,
and the agent passport proposal (APS)~\cite{pidlisnyi2026aps} explicitly models
delegation as a rooted tree with a single parent identifier, for which subtree
cascade is the intended mechanism.

\paragraph{Deployer-scoped cascade.}
Deployer-scoped cascade performs a forward traversal only through nodes visible
to the initiating deployer. It can miss cross-domain descendants whose
authority depends on the revoked edge. This models the common deployment reality
that revocation is administered per administrative domain: deleting a
Kubernetes \texttt{RoleBinding} affects only its namespace~\cite{k8srbac2024},
OAuth token revocation is performed at the issuing authorization
server~\cite{RFC7009}, and SPIFFE/SPIRE trust domains issue and rotate
credentials within a single domain~\cite{spiffe2018}. None of these provides a
cross-domain edge-revocation target.

\paragraph{Node or deployer revocation.}
Node revocation invalidates all authority held by one node, regardless of which
parent edge granted it---the model behind credential-centric revocation such as
CRLs~\cite{RFC5280} and OCSP~\cite{RFC6960}. Deployer revocation invalidates an
entire deployer domain. These mechanisms are appropriate for different operator
intents, such as key compromise, but they are coarse approximations for removing
one delegation relationship.

\paragraph{Segment revocation.}
Segment revocation targets content-addressed output segments, not delegation
authority. It is a useful primitive for output binding, but it is not evaluated
as a substitute for edge revocation.

\subsection{Error Metrics}

\begin{definition}[Over-revocation and under-revocation]
For mechanism $M$ and semantic target $S$:
\begin{align*}
  \overrev(M,G,S)&=|M(G)\setminus S|,\\
  \underrev(M,G,S)&=|S\setminus M(G)|.
\end{align*}
\end{definition}

Over-revocation is an availability and correctness failure: legitimate agents
lose authority. Under-revocation is a security failure: agents that should lose
authority remain active.

\subsection{Batch API Warning: Revocation Does Not Compose by Union}
\label{sec:batch}

Revocation APIs often process events independently. In multi-parent DAGs, that
implementation shortcut is not generally correct.

\begin{proposition}[Batch target for APIs]
For a set of revoked edges $R \subseteq E$:
\[
  \intent(R,G)=\reach(G)\setminus\reach(G\setminus R).
\]
In general,
\[
  \intent(R,G) \neq \bigcup_{e\in R}\intent(e,G).
\]
\end{proposition}

\begin{proof}
Let root delegate independently to $A$ and $B$, and let both $A$ and $B$
delegate to $C$. Let $R=\{A\to C,\;B\to C\}$. Removing either edge alone does
not disconnect $C$, so both single-edge targets are empty. Removing both edges
disconnects $C$, so $\intent(R,G)$ contains $C$. Therefore the batch target is
not the union of single-edge targets.
\end{proof}

This is an operational warning for agent revocation APIs. Even a system that
computes exact single-edge targets can under-revoke if it processes simultaneous
revocations by unioning independently computed single-edge targets.

\section{Measurement Methodology}

Figure~\ref{fig:pipeline} summarizes the \method{} measurement pipeline:
framework events are normalized into signed delegation DAGs and replayed under
exact and baseline target languages.

\begin{figure*}[t]
\centering
\includegraphics[width=0.88\textwidth]{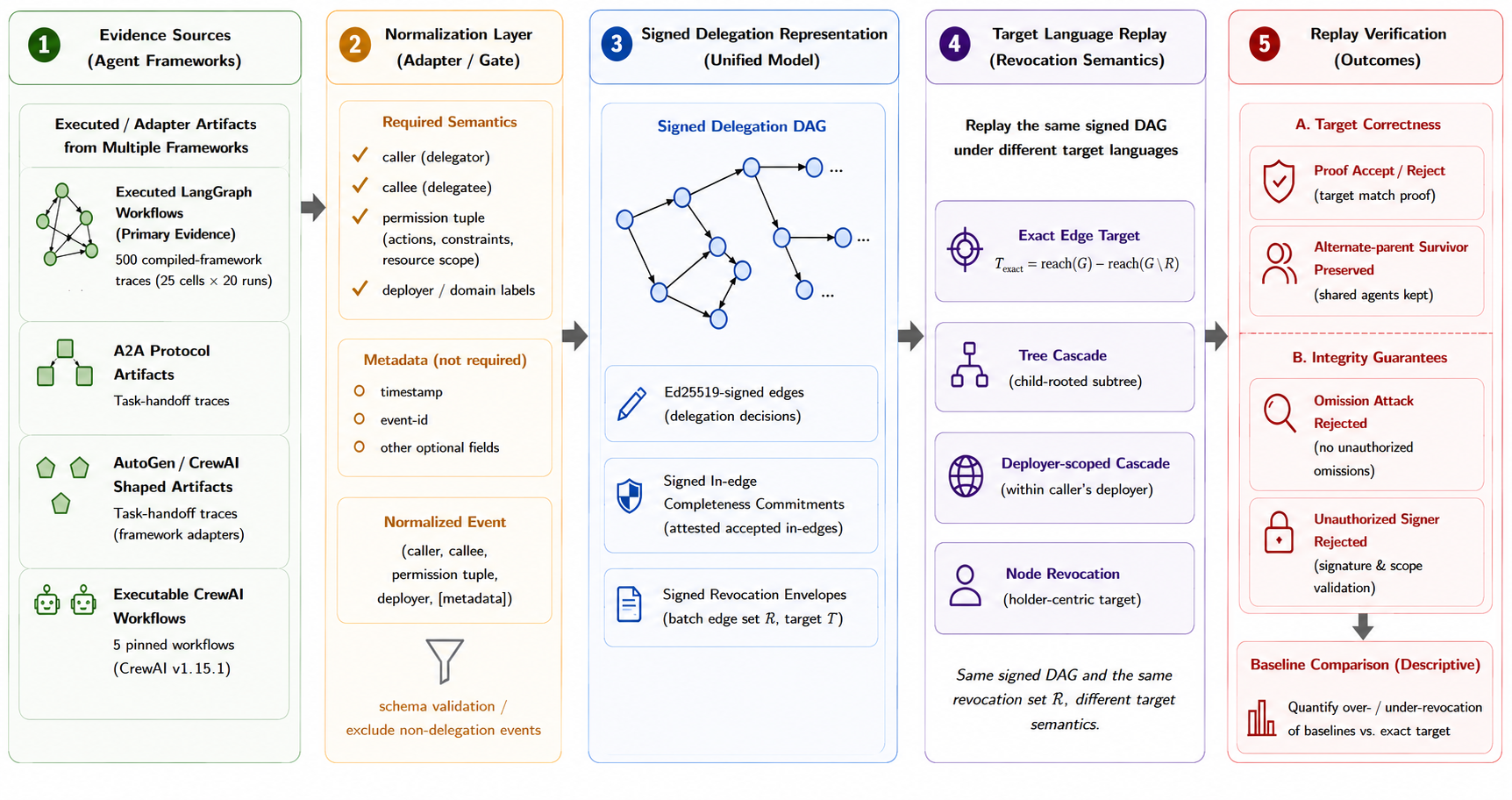}
\caption{\method{} replay pipeline. Framework evidence is normalized into a
signed DAG, replayed under exact and baseline target languages, and checked for
target correctness and integrity; E16 is primary and E17--E19 test portability.}
\label{fig:pipeline}
\end{figure*}

\subsection{Graph Families}

We combine deterministic witnesses, random and scale-free DAGs, batch-revocation
tests, executed LangGraph traces (E16), A2A/AutoGen/CrewAI portability artifacts
(E17--E19), and a static LangGraph corpus. E16 is the primary agent-native
evidence; mined and non-agent DAGs are supporting topology evidence rather than
direct revocation traffic.

\section{Results}

\subsection{RQ1: Witness Failures}

On the shared-subDAG witness root $\to A,B$; $A,B\to C$; $C\to D$, the
semantic target for revoking $A\to C$ is empty because $C$ and $D$ remain
reachable through $B\to C$. Tree cascade revokes $\{C,D\}$, producing
$\overrev=2$.

On the cross-deployer witness root$(d_1)\to A(d_1)\to C(d_1)\to D(d_2)$, the
semantic target for revoking $A\to C$ is $\{C,D\}$. A deployer-$d_1$-scoped
cascade revokes only $\{C\}$, producing $\underrev=1$.

These witnesses are small by design. Their role is not to establish frequency;
it is to show that the two failure modes are deterministic consequences of
topology and visibility assumptions.

\subsection{RQ2: Scale-Graph Stress Tests}

Generated DAGs are stress tests, not prevalence estimates. In
Barabasi--Albert~\cite{barabasi1999emergence}
DAGs with $n\in\{100,500,2000\}$ and attachment parameter $m\in\{2,4\}$, tree
cascade over-revokes on every sampled revocation event; mean collateral ranges
from $5.9$ nodes ($n=100,m=2$) to $79.7$ nodes ($n=2000,m=4$). Erdos--Renyi
checks with matched average degree show the same qualitative behavior. To rule
out cherry-picked edges, we also sample 200 random multi-parent DAGs under five
edge-selection policies (uniform, into-multi-parent, root-adjacent,
deepest-parent, largest-cascade); tree cascade over-revokes on $63.5$--$100\%$
of selected edges.

\subsection{RQ3: Batch Revocation}

The RBAC-inspired case study contains a backend service account with two parent
bindings. Revoking either parent edge alone leaves an alternate path, so both
single-edge targets are empty. Revoking both parent edges simultaneously
disconnects the backend node and its workers, producing a three-node batch
target. A system that unions independently computed single-edge targets would
revoke nothing.

We repeated this test systematically. In 200 random multi-parent DAGs with 80
nodes and shared-parent ratio 0.45, we sampled up to 200 candidate revocation
sets per graph from 2--4 incoming edges of multi-parent nodes. Across 39,991
sampled edge sets, the union of exact single-edge targets under-revoked the true
batch target in a mean 21.25\% of cases per graph (bootstrap 95\% CI
$[20.86,21.65]$; mean missed nodes 0.24; maximum 5). The gap is not universal:
many sampled batches are already captured by one constituent edge target. The
important point is that a nonzero, reproducible under-revocation gap remains
even when every single-edge target is individually exact.

\begin{table}[t]
\centering
\caption{Batch revocation cannot be implemented as the union of exact
single-edge targets. Random DAGs sample 2--4-edge incoming sets; the RBAC-inspired
rows are concrete multi-parent bindings from the case study.}
\label{tab:batch}
\resizebox{\columnwidth}{!}{%
\begin{tabular}{@{}lrrrp{0.23\columnwidth}@{}}
\toprule
Setting & Cases & Failure\% & Mean missed & Max missed \\
\midrule
Random multi-parent DAGs & 39,991 & 21.25 & 0.24 & 5 \\
RBAC: backend-sa parents & 1 & 100.0 & 3.00 & 3 \\
RBAC: ci-bot parents & 1 & 100.0 & 3.00 & 3 \\
RBAC: ci-step-deploy parents & 1 & 100.0 & 1.00 & 1 \\
\bottomrule
\end{tabular}}
\end{table}

We treat batch non-compositionality as a correctness corner case rather than a
prevalence claim. In the acyclic subset of the mined LangGraph corpus, 8/33
alternate-parent pairs exhibit it, but its operational significance is clearer
than its frequency: it is exactly the case that arises when alternate parents are
revoked together during scope teardown, parent-key compromise, or policy
rollback. The target language must support batch edge sets natively and compute
$\reach(G)\setminus\reach(G\setminus R)$ directly; otherwise, even a correct
single-edge implementation can under-revoke during incident response.

\subsection{RQ4: Agent-Framework Trace Replay}
\label{sec:real}

The witnesses and synthetic tests could be dismissed as generator artifacts, so
our primary framework-grounded execution test uses executed agent workflows. E16
runs seven fixed LangGraph workflow families plus an exhaustive $3{\times}6$
customer-support matrix, repeating each workflow/input cell 20 times. Each
execution is a compiled \texttt{StateGraph}; we record LangGraph runtime-stream
events and application-level caller-to-callee invocations, then normalize them
into parent-domain-signed delegation decisions. Construct validity is deliberately
narrow: a raw runtime event becomes a delegation decision only when it records a
caller, callee, permission tuple (tenant, resource, action, constraints), and
known deployer labels; example retained fields are caller, callee, tenant,
resource, action, constraints, and deployer. Generic framework stream events are
excluded.

This yields 500 framework traces, 10,500 raw runtime events, and 2,000/2,000
valid signed delegation decisions across 20 deployer domains. Topology claims use
the 25 workflow/input cells as clusters, not the deterministic repeats: 13/25
cells contain a runtime-observed multi-parent agent and 8/25 contain
cross-deployer sharing. Cell-level tree-cascade over-revocation affects a mean
46.9\% of edges (bootstrap 95\% CI $[29.3,64.8]$); deployer-scoped
under-revocation affects 61.6\% ($[45.6,76.7]$).

Signed replay accepts 500/500 target proofs and 500/500 authorized revocations,
including 160/160 cross-deployer cases, and rejects 500/500 unauthorized signers
and 260/260 omission attacks. For 320 selected single-edge withdrawals with an
alternate runtime parent, the edge target preserves the shared agent in all 320
cases while tree cascade revokes it in all 320. The repeated executions test
adapter and replay stability; topology claims are reported at the clustered
workflow/input-cell level.

We also replay E16 against target-language baselines. Over 1,960 single-edge
withdrawals, edge targets and complete-graph credential revalidation are exact in
1,960/1,960 cases, but the latter requires the verifier to hold the complete
graph. Holder/node revocation preserves 0/320 alternate-parent agents; APS-style
tree cascade also preserves 0/320 and is exact for only 35.7\% of edge cases.
Deployer-scoped cascade is exact for 49.0\% and under-revokes all 880 cases whose
semantic target crosses a deployer boundary.

E17--E19 check portability beyond LangGraph at the schema/signature and
target-replay level: A2A, AutoGen, and CrewAI artifacts emit nine traces and
53/53 valid signed delegation events through the same schema. E19 runs five
pinned CrewAI~1.15.1 workflows with deterministic local LLM behavior; all 5/5
\texttt{Crew.kickoff()} runs complete, yielding 35/35 valid task-context
handoffs. Across E17--E19, edge targets preserve 8/8
alternate-parent cases and deployer-scoped cascade misses 19/19 cross-domain
target cases.

We also mine 28 third-party LangGraph applications using recorded GitHub queries,
repository commits, and an exclusion policy that removes tests, examples,
tutorials, notebooks, and generated output. Static extraction yields 47 graphs
(382 nodes, 384 edges), of which 31 contain a multi-parent node. Because 29
extracted control-flow graphs contain cycles, DAG-semantics replay is restricted
to the 18 acyclic graphs (123 nodes, 102 edges): 8/18 contain a multi-parent
node, tree cascade deviates on 54.9\% of edges, and 8/33 alternate-parent pairs
are batch non-compositional. We use the corpus as lower-bound structural
evidence; cyclic graphs require epoch unrolling before target replay.

\begin{table}[t]
\centering
\caption{Agent-native, baseline, portability, and topology evidence. E16 replay
is primary; E17--E19 check schema portability; the mined corpus is static
lower-bound evidence.}
\label{tab:realtraces}
\resizebox{\columnwidth}{!}{%
\begin{tabular}{lllll}
\toprule
Source & Units & MP/AP & X-dep unit & Main outcome \\
\midrule
E16 design & 25 cells, 500 traces & 13 MP & 8 cells & repeats test stability \\
E16 signed replay & 2,000 decisions & 320 AP & 160 cases & 500/500 proofs verified \\
E16 mutations & 760 attacks & -- & -- & 760/760 rejected \\
E16 baselines & 1,960 edges & 320 AP & 880 targets & tree 0/320 AP; scoped 49.0\% exact \\
E17 A2A artifact & 2 traces, 8 decisions & 1 AP & 6 targets & schema/signature valid \\
E18 shaped artifacts & 2 traces, 10 decisions & 2 AP & 3 targets & schema/signature valid \\
E19 CrewAI executable & 5 traces, 35 decisions & 5 AP & 10 targets & kickoff + schema valid \\
OSS LangGraph corpus & 47 static; 18 DAGs & 31 MP; 8 DAG MP & -- & 54.9\% DAG tree deviation \\
\bottomrule
\end{tabular}}
\end{table}

\section{Prototype and Overhead}
\label{sec:prototype}

We implement \method{} in Python with \texttt{cryptography} Ed25519. The
\method{} API and CLI load normalized events, issue edge revocations, export
proof bundles, and verify them without the original graph.

\paragraph{Framework adapter boundary.}
\method{} does not require replacing LangGraph, CrewAI, A2A, or MCP. An
adapter needs four hooks: emit caller-to-callee edges, sign accepted-in-edge
commitments, sign parent revocation envelopes, and verify cached edge sets or
proof bundles. In existing frameworks, the hooks sit in middleware or callbacks
rather than a fork (Table~\ref{tab:integration}).

\begin{table}[t]
\centering
\caption{Adapter boundary. ``No'' means no framework fork or source patch;
wrapper/callback evidence is still emitted.}
\label{tab:integration}
\resizebox{\columnwidth}{!}{%
\begin{tabular}{llll}
\toprule
Target & Hook surface & Framework fork? & Evidence \\
\midrule
LangGraph & stream + app calls & No & 500 traces \\
A2A/AutoGen/CrewAI-shaped & wrappers & No & 4 traces \\
CrewAI executable & callbacks & No & 5 workflows \\
\method{} API/CLI & normalized JSONL & n/a & offline verify \\
\bottomrule
\end{tabular}}
\end{table}

\paragraph{Signed delegation graph.}
Every edge is signed by the parent node's deployer key. A relying party pins
deployer public keys out of band and authenticates each edge locally
(Section~\ref{sec:threat}). Edge establishment is bilateral and is part of the
protocol, not a free property: a delegation edge becomes valid only after the
parent deployer signs the edge descriptor \emph{and} the child deployer updates
and signs its accepted-in-edge commitment to include the new edge. This second
signature is what later makes adjacency completeness checkable; a node that never
accepted an edge will not have committed to it.

\paragraph{Rollback-evident revocation log.}
Revocation events are edge-targeted, signed envelopes appended to a log that
publishes a signed checkpoint containing the entry count, a history hash, and the
previous checkpoint hash. A relying party that has observed checkpoint $c$
rejects a later log with a smaller count, a substituted same-count history, or a
broken checkpoint chain. In the prototype, an honest extended log is accepted and
rollback/substitution is detected.

\paragraph{Proof-carrying target set with completeness.}
Rather than ship the whole graph, an issuer can ship the claimed target $S$ plus
\emph{survivor witnesses}: for each reachable node not in $S$, one
revocation-free root path of signed edges (proving it survives). Disclosing only
the in-edges of target nodes is \emph{not} sound on its own: a malicious issuer
could omit an unrevoked in-edge to make a live node look disconnected. We close
this with \emph{authenticated adjacency completeness}. Because a node accepts
each delegation offered to it, the node's own deployer maintains and signs a
commitment to that node's \emph{complete} accepted in-edge set. The proof
includes this signed commitment for every target node, and the verifier accepts
$S$ iff (i)~every survivor witness authenticates and is revocation-free, (ii)~no
target node appears as a survivor, (iii)~the disclosed cut in-edges of each
target node \emph{exactly match} its signed commitment (no omission), and
(iv)~every unrevoked in-edge of a target node originates from another target
node. Under the same trust assumption as edge authenticity---a deployer does not
sign evidence that contradicts the delegations it accepted---omission is
detectable. In an adversarial test that forges a target by revoking one in-edge
and omitting a surviving alternate parent, the verifier rejects all $40/40$
attempts.

Completeness of the node \emph{set} itself is pinned separately. The
federation-signed snapshot manifest commits to the authorization roster
$B=\reach(G)$ of all root-reachable nodes, and the verifier accepts a target set $T$
with survivor set $W$ only if $T\cap W=\emptyset$ and $T\cup W=B$. A genuine
reachable node therefore cannot be silently omitted from the proof: dropping it
breaks $T\cup W=B$ and the proof is rejected. This is the realization of
Proposition~\ref{prop:localsafety}(iii), and is exercised by the adversarial unit
test that removes a disconnected roster node and confirms rejection.

\paragraph{Overhead.}
On generated multi-parent DAGs ($45\%$ shared-parent ratio; medians over 30 graph
seeds), target computation is linear and takes $0.17$--$2.15$ ms up to 2,000
nodes (Table~\ref{tab:proto}). The dominant cost is Ed25519 verification, not
graph traversal: cold proof verification remains below full-view verification,
and the proof bundle is $64.7$--$69.0\%$ of the full signed bundle, saving
roughly one third of graph disclosure. This control-plane cost is paid when a
revocation bundle or checkpoint is ingested, not on each agent tool call; hot-path
authorization uses cached verified revocation state.

\begin{table}[t]
\centering
\caption{Prototype overhead (median over 30 generated graphs).}
\label{tab:proto}
\resizebox{\columnwidth}{!}{%
\begin{tabular}{@{}rrrrr@{}}
\toprule
Nodes & Target ms & Proof verify ms & Full view ms & Proof/full bytes \\
\midrule
200 & 0.17 & 45.9 & 61.0 & 64.7\% \\
500 & 0.43 & 113.4 & 155.3 & 66.0\% \\
1,000 & 0.91 & 229.7 & 310.8 & 67.0\% \\
2,000 & 2.15 & 463.6 & 644.3 & 69.0\% \\
\bottomrule
\end{tabular}}
\end{table}

The attack matrix rejects 1,100/1,100 malformed envelopes, commitments, Merkle
positions, log histories, and unauthorized revocation signers. E16 adds 760/760
replay-level rejection tests for unauthorized signers and omitted in-edges.

The prototype assumes observed revocation events; once observed, targets are
cheap to compute and check from signed evidence, and log rollback is detectable.

\section{Design Implications}

For AI-agent framework builders, the implication is an API boundary rather than
a new orchestrator: name delegation edges, expose accepted-in-edge commitments,
carry graph evidence, handle batches natively as
$S=\reach(G)\setminus\reach(G\setminus R)$, and keep target correctness separate
from delivery.

\section{Limitations}

The policy is disjunctive authority; all-parent, $k$-of-$n$, or path-scoped
authority need different targets. E16 uses real LangGraph machinery but
controlled workloads, E17--E19 test portability rather than prevalence, and the
static corpus can miss dynamic routing. Production-scale traces and propagation
remain open. A propagation layer should budget stale authority, e.g.,
$v(\mathrm{TTL}+\delta_{\mathsf{prop}})$ operations for rate $v$; this paper
fixes the target such a layer must deliver.

\section{Related Work}

\paragraph{Delegation graphs and revocation semantics.}
Delegated authority has long been represented as graph-structured evidence in
trust management~\cite{blaze1996trust,clarke2001spki,halpern2002spki,li2002rt}.
Revocation work classifies cascades and centralized workflow
policies~\cite{hagstrom2001revocation,wainer2007delegation}. We study signed,
multi-parent agent DAGs checked verifier-independently.

\paragraph{Agent identity and delegated authority.}
Recent agent-identity work treats delegation as a first-class AI-agent concern:
authenticated agent authority~\cite{south2025delegation}, lifecycle management
~\cite{south2025identity}, and verifiable delegation across MCP and A2A
~\cite{prakash2026aip}. These systems authenticate delegation; we define the
authority loss caused by revoking one edge.

\paragraph{Agent protocols and access-control enforcement.}
AutoGen, CrewAI, and LangGraph compose multi-agent workflows
~\cite{wu2023autogen,crewai2024,langgraph2026}; MCP and A2A define protected
tool authorization and cross-agent communication~\cite{mcp2025auth,google2026a2a}.
Runtime supply-chain and protocol analyses identify agentic attack surfaces and
lifecycle risks~\cite{jiang2026agentsurface,anbiaee2026agentprotocols}. Agent
access-control systems govern information flow, actions, or capabilities
~\cite{li2025aac,ji2026seagent,jiang2026chaincaps}, and compositional governance
formalizes delegation and scope~\cite{ibrahim2026governance}. We isolate the
revocation target and evidence for a withdrawn signed edge.

\paragraph{Credentials, capabilities, and tree passports.}
CRLs~\cite{RFC5280}, OCSP~\cite{RFC6960}, CRLite~\cite{crlite2017},
Macaroons~\cite{birgisson2014macaroons}, and SPIFFE/SPIRE~\cite{spiffe2018}
revoke credentials/tokens or bound authority with short lifetimes. APS is closest
among agent proposals but uses a single-parent rooted tree~\cite{pidlisnyi2026aps}.
None directly provides a multi-parent, cross-domain edge target
(Table~\ref{tab:mechanisms}).

\paragraph{Dominators.}
Our target is a root-relative edge-dominance set, computable with standard
dominator algorithms~\cite{lengauer1979dominators}. The
contribution is using that set as the revocation contract for signed agent
delegation.

\section{Conclusion}

AI-agent revocation is an authority-consistency problem in the workflow runtime,
not only a token lifecycle problem. In a federated multi-parent workflow,
revoking one delegation edge need not revoke the child agent, while a local
cascade can miss cross-deployer descendants that still depend on the withdrawn
relationship. This paper makes the runtime contract explicit: under disjunctive
authority, edge revocation invalidates exactly
$\reach(G)\setminus\reach(G\setminus\{e\})$. Used as the reference target, this
contract exposes over-revocation for tree cascades, under-revocation for
deployer-scoped cascades, and non-compositional behavior for batch withdrawals.
The practical implication for agent tools is concrete: \method{} adapters should
emit signed delegation edges, revocation APIs should accept edge sets, and
relying parties should receive enough authenticated graph evidence to check the
authority loss they are asked to enforce.

\noindent\textbf{Artifact Availability.}
The artifact is available at
(\href{https://anonymous.4open.science/r/ictai2026-agent-revocation-artifact-14BA/README.md}{\texttt{artifact-14BA}}).
It includes the \method{} Python API and CLI, E16--E19 traces and scripts,
artifact checks, and reproduction instructions.

\bibliographystyle{IEEEtran}
\bibliography{references}

@techreport{RFC5280,
  author       = {D. Cooper and S. Santesson and S. Farrell and
                  S. Boeyen and R. Housley and W. Polk},
  title        = {{Internet X.509 Public Key Infrastructure Certificate
                   and Certificate Revocation List (CRL) Profile}},
  type         = {RFC},
  number       = {5280},
  institution  = {IETF},
  year         = {2008},
  doi          = {10.17487/RFC5280},
}

@techreport{RFC6960,
  author       = {S. Santesson and M. Myers and R. Ankney and
                  A. Malpani and S. Galperin and C. Adams},
  title        = {{X.509 Internet Public Key Infrastructure Online
                   Certificate Status Protocol -- OCSP}},
  type         = {RFC},
  number       = {6960},
  institution  = {IETF},
  year         = {2013},
  doi          = {10.17487/RFC6960},
}

@inproceedings{crlite2017,
  author    = {James Larisch and David Choffnes and Dave Levin and
               Bruce M. Maggs and Alan Mislove and Christo Wilson},
  title     = {{CRLite}: A Scalable System for Pushing All {TLS}
               Revocations to All Browsers},
  booktitle = {Proceedings of the 2017 IEEE Symposium on Security and
               Privacy (S\&P)},
  year      = {2017},
  pages     = {539--556},
  doi       = {10.1109/SP.2017.17},
}

@inproceedings{birgisson2014macaroons,
  author    = {Arnar Birgisson and Joe Gibbs Politz and {\'U}lfar Erlingsson
               and Ankur Taly and Michael Vrable and Mark Lentczner},
  title     = {Macaroons: Cookies with Contextual Caveats for Decentralized
               Authorization in the Cloud},
  booktitle = {Proceedings of NDSS},
  year      = {2014},
}

@techreport{spiffe2018,
  author      = {CNCF SPIFFE Working Group},
  title       = {{SPIFFE}: Secure Production Identity Framework For Everyone},
  institution = {Cloud Native Computing Foundation},
  year        = {2018},
  note        = {\url{https://spiffe.io/}},
}

@techreport{RFC7009,
  author      = {T. Lodderstedt and S. Dronia and M. Scurtescu},
  title       = {{OAuth 2.0 Token Revocation}},
  type        = {RFC},
  number      = {7009},
  institution = {IETF},
  year        = {2013},
  doi         = {10.17487/RFC7009},
}

@inproceedings{blaze1996trust,
  author    = {Matt Blaze and Joan Feigenbaum and Jack Lacy},
  title     = {Decentralized Trust Management},
  booktitle = {Proceedings of the 1996 IEEE Symposium on Security and Privacy},
  year      = {1996},
  pages     = {164--173},
  doi       = {10.1109/SECPRI.1996.502679},
}

@article{clarke2001spki,
  author  = {Dwaine Clarke and Carl Ellison and Matt Fredette and
             Alexander Morcos and Ronald L. Rivest},
  title   = {Certificate Chain Discovery in {SPKI/SDSI}},
  journal = {Journal of Computer Security},
  volume  = {9},
  number  = {4},
  pages   = {285--322},
  year    = {2001},
}

@misc{halpern2002spki,
  author       = {Joseph Y. Halpern and Ron van der Meyden},
  title        = {A Logical Reconstruction of {SPKI}},
  howpublished = {arXiv:cs/0208028},
  year         = {2002},
}

@inproceedings{li2002rt,
  author    = {Ninghui Li and John C. Mitchell and William H. Winsborough},
  title     = {Design of a Role-Based Trust-Management Framework},
  booktitle = {Proceedings of the 2002 IEEE Symposium on Security and Privacy},
  year      = {2002},
  pages     = {114--130},
  doi       = {10.1109/SECPRI.2002.1004366},
}

@article{wainer2007delegation,
  author  = {Joel Wainer and Akhil Kumar and Paulo Barthelmess},
  title   = {{DW-RBAC}: A formal security model of delegation and
             revocation in workflow systems},
  journal = {Information Systems},
  volume  = {32},
  number  = {3},
  pages   = {365--384},
  year    = {2007},
  doi     = {10.1016/j.is.2005.11.008},
}

@inproceedings{hagstrom2001revocation,
  author    = {{\AA}sa Hagstr{\"o}m and Sushil Jajodia and Francesco
               Parisi-Presicce and Duminda Wijesekera},
  title     = {Revocations --- a classification},
  booktitle = {Proceedings of the 14th IEEE Computer Security Foundations
               Workshop (CSFW)},
  year      = {2001},
  pages     = {44--58},
  doi       = {10.1109/CSFW.2001.930135},
}

@misc{south2025delegation,
  author       = {Tobin South and others},
  title        = {Authenticated Delegation and Authorized {AI} Agents},
  howpublished = {arXiv:2501.09674},
  year         = {2025},
  doi          = {10.48550/arXiv.2501.09674},
}

@misc{south2025identity,
  author       = {Tobin South and others},
  title        = {Identity Management for Agentic {AI}: The New Frontier of
                  Authorization, Authentication, and Security for an {AI} Agent
                  World},
  howpublished = {OpenID Foundation whitepaper; arXiv:2510.25819},
  year         = {2025},
  doi          = {10.48550/arXiv.2510.25819},
}

@misc{prakash2026aip,
  author       = {Sunil Prakash},
  title        = {{AIP}: Agent Identity Protocol for Verifiable Delegation Across
                  {MCP} and {A2A}},
  howpublished = {arXiv:2603.24775},
  year         = {2026},
  doi          = {10.48550/arXiv.2603.24775},
}

@misc{pidlisnyi2026aps,
  author       = {Tymofii Pidlisnyi},
  title        = {Agent Passport System ({APS})},
  howpublished = {\url{https://datatracker.ietf.org/doc/draft-pidlisnyi-aps/}},
  year         = {2026},
  month        = feb,
  note         = {IETF Internet-Draft, work in progress},
}

@misc{google2026a2a,
  author       = {{A2A Project}},
  title        = {Agent2Agent ({A2A}) Protocol Specification},
  howpublished = {\url{https://a2a-protocol.org/latest/specification/}},
  year         = {2025},
}

@misc{mcp2025auth,
  author       = {{Model Context Protocol}},
  title        = {Authorization},
  howpublished = {\url{https://modelcontextprotocol.io/specification/2025-11-25/basic/authorization}},
  year         = {2025},
  note         = {Specification revision 2025-11-25},
}

@misc{anbiaee2026agentprotocols,
  author       = {Zeynab Anbiaee and others},
  title        = {Security Threat Modeling for Emerging {AI}-Agent Protocols:
                  A Comparative Analysis of {MCP}, {A2A}, {Agora}, and {ANP}},
  howpublished = {arXiv:2602.11327},
  year         = {2026},
  doi          = {10.48550/arXiv.2602.11327},
}

@misc{li2025aac,
  author       = {Xinfeng Li and others},
  title        = {A Vision for Access Control in {LLM}-based Agent Systems},
  howpublished = {arXiv:2510.11108},
  year         = {2025},
  doi          = {10.48550/arXiv.2510.11108},
}

@misc{ji2026seagent,
  author       = {Zimo Ji and others},
  title        = {Taming Various Privilege Escalation in {LLM}-Based Agent
                  Systems: A Mandatory Access Control Framework},
  howpublished = {arXiv:2601.11893},
  year         = {2026},
  doi          = {10.48550/arXiv.2601.11893},
}

@misc{ibrahim2026governance,
  author       = {Amjad Ibrahim and Yong Li},
  title        = {Overlaying Governance: A Compositional Authorization Framework
                  for Delegation and Scope in Agentic {AI}},
  howpublished = {arXiv:2606.03518},
  year         = {2026},
  doi          = {10.48550/arXiv.2606.03518},
}

@inproceedings{jiang2026agentsurface,
  author    = {Xiaochong Jiang and Shiqi Yang and Wenting Yang and Yichen Liu
               and Cheng Ji},
  title     = {Agentic {AI} as a Cybersecurity Attack Surface: Threats,
               Exploits, and Defenses in Runtime Supply Chains},
  booktitle = {Proceedings of the 2026 IEEE Conference on Artificial
               Intelligence (CAI)},
  year      = {2026},
  pages     = {2142--2149},
  doi       = {10.1109/CAI68641.2026.11536564},
}

@misc{jiang2026chaincaps,
  author       = {Xiaochong Jiang and others},
  title        = {{ChainCaps}: Composition-Safe Tool-Using Agents via Monotonic
                  Capability Attenuation},
  howpublished = {arXiv:2605.26542},
  year         = {2026},
  doi          = {10.48550/arXiv.2605.26542},
}

@misc{wu2023autogen,
  author       = {Qingyun Wu and others},
  title        = {{AutoGen}: Enabling Next-Gen {LLM} Applications via
                  Multi-Agent Conversation},
  howpublished = {arXiv:2308.08155; Oral, ICLR 2024 Workshop on LLM Agents},
  year         = {2023},
}

@misc{langgraph2026,
  author       = {{LangChain}},
  title        = {{LangGraph}: Graph {API}},
  howpublished = {\url{https://docs.langchain.com/oss/python/langgraph/graph-api}},
  year         = {2026},
}

@misc{crewai2024,
  author       = {{CrewAI, Inc.}},
  title        = {{CrewAI}: Framework for Orchestrating Role-Playing,
                  Autonomous {AI} Agents},
  howpublished = {\url{https://crewai.com}},
  year         = {2024},
}

@article{lengauer1979dominators,
  author  = {Thomas Lengauer and Robert Endre Tarjan},
  title   = {A Fast Algorithm for Finding Dominators in a Flowgraph},
  journal = {ACM Transactions on Programming Languages and Systems},
  volume  = {1},
  number  = {1},
  pages   = {121--141},
  year    = {1979},
  doi     = {10.1145/357062.357071},
}

@misc{k8srbac2024,
  author       = {{The Kubernetes Authors}},
  title        = {Using {RBAC} Authorization},
  howpublished = {Kubernetes Documentation},
  year         = {2024},
}

@article{barabasi1999emergence,
  author  = {Albert-L\'{a}szl\'{o} Barab\'{a}si and R\'{e}ka Albert},
  title   = {Emergence of Scaling in Random Networks},
  journal = {Science},
  volume  = {286},
  number  = {5439},
  pages   = {509--512},
  year    = {1999},
  doi     = {10.1126/science.286.5439.509},
}

\end{document}